\def\year{2022}\relax
\pdfoutput=1

\documentclass[letterpaper]{article} 
\usepackage{aaai22}  
\usepackage{times}  
\usepackage{helvet}  
\usepackage{courier}  
\usepackage[hyphens]{url}  
\usepackage{xcolor}
\usepackage{graphicx} 
\usepackage{natbib}  
\usepackage{caption} 
\DeclareCaptionStyle{ruled}{labelfont=normalfont,labelsep=colon,strut=off} 
\usepackage{algorithm}
\usepackage{algorithmic}
\usepackage{amsmath}
\usepackage{ amssymb }

\usepackage{dsfont}
\usepackage[english]{babel}
\usepackage{graphicx}
\usepackage{hyperref}
\usepackage{thmtools, thm-restate}

\newenvironment{proof}{\paragraph{Proof:}}{\hfill$\square$}

\usepackage{newfloat}
\usepackage{listings}
\floatstyle{ruled}
\newfloat{listing}{tb}{lst}{}
\floatname{listing}{Listing}
\title{MiddleGAN: Generate Domain Agnostic Samples for Unsupervised Domain Adaptation}

\author{
    Ye Gao\textsuperscript{\rm 1},
    Zhendong Chu \textsuperscript{\rm 1},
    Hongning Wang \textsuperscript{\rm 1},
    John Stankovic \textsuperscript{\rm 1}
}
\affiliations{
    \textsuperscript{\rm 1}Department of Computer Science, University of Virginia, Charlottesville, United States}

\usepackage{bibentry}

\begin{document}

\maketitle

\begin{abstract}
In recent years, machine learning has achieved impressive results across different application areas. However, machine learning algorithms do not necessarily perform well on a new domain with a different distribution than its training set. Domain Adaptation (DA) is used to mitigate this problem. One approach of existing DA algorithms is to find domain invariant features whose distributions in the source domain are the same as their distribution in the target domain. In this paper, we propose to let the classifier that performs the final classification task on the target domain learn implicitly the invariant features to perform classification. It is achieved via feeding the classifier during training generated fake samples that are similar to samples from both the source and target domains. We call these generated samples domain agnostic samples. To accomplish this we propose a novel variation of generative adversarial networks (GAN), called the MiddleGAN, that generates fake samples that are similar to samples from both the source and target domains, using two discriminators and one generator. We extend the theory of GAN to show that there exist optimal solutions for the parameters of the two discriminators and one generator in MiddleGAN, and empirically show that the samples generated by the MiddleGAN are similar to both samples from the source domain and samples from the target domain. We conducted extensive evaluations using 24 benchmarks; on the 24 benchmarks, we compare MiddleGAN against various state-of-the-art algorithms and outperform the state-of-the-art by up to 20.1\% on certain benchmarks.
\end{abstract}

\section{Introduction} 
In recent years, deep learning has achieved impressive results across different application domains \cite{esteva2021deep, he2016deep, szegedy2017inception, zhu2017densenet, purwins2019deep, noda2015audio, wu2021deep, wahab2021dna}. However, a deep neural net does not necessarily perform well on a new domain with different distribution than its training set. This problem is called domain shift, and domain adaptation (DA) have been invented to tackle the issue of domain shift. One approach of DA is to find domain-invariant features \cite{zhao2019learning}. 

Instead of explicitly selecting domain-invariant features, we propose to let a classifier that will perform the classification task on the target domain implicitly learn to use domain-invariant features (to perform classification). In this way, we do not have to hand-engineer the features which may not be inclusive enough to include all the features that are domain-invariant. Our intuition is based on the observation that deep neural networks such as the ResNet-50 \cite{he2016deep} or Inception \cite{szegedy2016rethinking} generalize well when trained on a large amount of data. If we want the classifier to learn the domain invariant features (implicitly), we need a large amount of samples that is similar to both the source domain samples and the target domain samples. We call those samples domain agnostic samples. If we train a neural network such as the ResNet-50 with a large quantity of domain agnostic samples, the neural network will implicitly learn to use the domain-invariant features to perform classification.

How do we generate those domain agnostic samples? We propose a variation of GAN, called the MiddleGAN, which has two discriminators and a generator. One discriminator is for the source domain; it tries to distinguish a generated sample from real samples from the source domain. Another discriminator is for the target domain; it tries to distinguish a generated sample from the real samples from the target domain. The generator is trying to generate samples that can deceive both discriminators at the same time. The three neural networks engage in a minimax game in which the generator is trying to generate samples to confuse both the source discriminator and the target discriminator. Ideally, after training, the generated samples will be indistinguishable from both the real source domain samples and the real target domain samples, thus achieving the similarity to samples of both domains. We have also extended the theory of GAN to theoretically prove that there exist optimal values for the source and target discriminators and the generator.

\textbf{The contributions of this paper are}:
\begin{itemize}
    \item We create a novel variation of GAN, the MiddleGAN, that generates samples that are similar to samples from both the source and target domains. In other words, these generated samples are domain invariant.
    \item We conduct extensive evaluation on 24 benchmarks; on the 24 benchmarks, we compare MiddleGAN against various state-of-the-art algorithms and outperform the state-of-the-art by up to 20.1\% on certain benchmarks.
\end{itemize}
\section{Related Works}
In this Section we present two important areas of research related to our work: the Generative Adversarial Nets (GANs), and recent advances in domain adaptation.

\subsection{Generative Adversarial Nets}
Goodfellow et al. \cite{goodfellow2014generative} propose the original GAN which consists of two neural nets: the discriminator and the generator. The two nets engage in a minimax game where the generator attempts to generate images from noise to fool the discriminator, while the discriminator attempts to distinguish generated images from real ones. 
Inspired by this work, works such as the CGAN \cite{mirza2014conditional} and ACGAN \cite{odena2017conditional} attempt to regulate the classes of the generated images. The CGAN \cite{mirza2014conditional} is constructed in the way that, during training, the label information is fed to both the generator and the discriminator. The discriminator of the ACGAN \cite{odena2017conditional} has two objective functions: to maximize the log likelihood that a given sample is of the correct source (generated or real), and to maximize the log likelihood that the label (which class is this sample from) of the sample is correct. Moving past the GANs that leverage label information, another set of GANs focus on cycle consistency of the generated images. For example, the CycleGAN \cite{zhu2017unpaired} employs two generators, one to translate an image from the source to the target and the other to translate back a translated image by the first generator. A cycle consistency loss is added to minimize the discrepancy between an original, unaltered image, and the image translated by the first generator and then translated back by the second generator. StarGAN \cite{choi2018stargan} addresses the scalability issue that different GAN models need to be created for all pairs of domains. Unlike the previous works \cite{mirza2014conditional, odena2017conditional} which use only one generator, the MiddleGAN employs two discriminators and one generator and aims to generate samples that are similar to both the source domain samples and the target domain samples. MiddleGAN is a GAN designed specifically for Domain Adaptation while the other previous works on GAN mentioned in this section seek to generate realistic samples or achieve style transfer.

\subsection{Domain Adaptation}
One trend in the area of domain adaptation is to find domain-invariant features. DANN \cite{ganin2016domain} proposes a new neural network architecture based on the traditional feed-forward architecture by adding extra layers and a gradient reversal layer. Through adversarial training, it is able to find domain-invariant features and correctly classify source and target samples based on these features. Another way to find domain invariant features is to reduce the Maximum Mean Discrepancy (MMD) \cite{gretton2012kernel} between the feature representations of the source samples and the feature representations of the target samples. Tzeng et al. \cite{tzeng2014deep} and Long et al. \cite{long2015learning} both reduce the MMD while training a model to perform well on the source domain. A variation of this trend is to find the association of the source features and target features so that the domain discrepancy on the learned features is reduced, such as Haeusser et al \cite{haeusser2017associative}. Another variation of this trend is to use a generator/encoder to encode target samples so they are projected to the source feature space and, therefore, can be classified by the source model \cite{tzeng2017adversarial}. Other works that focus on finding domain-invariant features that are discriminative to the classification task are: \cite{gopalan2011domain, sun2016deep, long2017deep, ganin2015unsupervised}. However, the aforementioned approaches fail to address the scenario in which the discrepancy of the source and target is very large because domain-invariant features will be harder to find. As a result, another new trend is to find the middle domains of the source and target domains. Fixbi \cite{na2021fixbi} proposes to establish one source-dominant domain and one-target dominant domain as two intermediate domains to help bridge the gap between the source and target domain. Unlike the MiddleGAN, Fixbi is not a generative model so it does not generate domain invariant samples using the two intermediate domains. Instead, it uses the two intermediate domains to train a source-dominant model and a target-dominant model via bi-directional feature matching. The bi-directional feature matching guarantees that the parameters of source-dominant model and target-dominant model converge. Other (unsupervised) domain adaptation algorithm include RSDA \cite{gu2020spherical}, which proposes an adversarial domain adaptation scheme in the spherical feature space, SRDC \cite{tang2020unsupervised}, which proposes to achieve unsupervised domain adaptation via structurally regularized deep clustering, CAN \cite{kang2019contrastive} proposes to explicitly model intra-class and inter-class domain discrepancy. SEMA \cite{zuo2021margin} attempts to address the issue that most domain adaptation algorithms ignore the discriminative features among classes. The Enforced Transfer \cite{gao2022enforced} is based on the idea that some target samples closer to the distribution of the source domain should be directly processed by the source classifier, instead of training a target classifier to process them.


\section{MiddleGAN}

Before we discuss our MiddleGAN, we need to discuss the original GAN on which MiddleGAN is based. In the original GAN \cite{goodfellow2014generative}, the generator $G$ and the discriminator $D$ engage in a minimax game in which $G$ tries to minimize a value objective $V(G,D)$ whereas $D$ tries to maximize it. $V(G,D)$ is defined in Equation \ref{eq:original_GAN}, in which $p$ is the distribution of the real samples and $q$ is the distribution of the noise. A key observation obtained from Equation \ref{eq:original_GAN} is that $G$'s effort is to generate $G({z})$ whereas ${z}$ is an input noise such that $G({z})$ will be in-distribution with the distribution of the real samples $p$.

\begin{equation} 
\label{eq:original_GAN}
\begin{split}
\underset{G}{\min} \; \underset{D}{\max} \; V(G,D) 
&= \mathds{E}_{{x} \sim p({x})} [\log(D({x})) \\
&+ \mathds{E}_{{z} \sim q({z})} [\log(1-D(G(z)))]
\end{split}
\end{equation}


Based on the key observation that we obtain from Equation \ref{eq:original_GAN}, in MiddleGAN we propose to employ two discriminators, $D_s$ and $D_t$. $D_s$ tries to distinguish a generated sample from real source domain samples, and $D_t$ tries to distinguish a generated sample from real target domain samples. The generator $G$ engages in a two-way minimax game with the two discriminators. The samples it generates will be in the middle of the feature space of the source and the target domains. Below, we empirically prove that the generated samples in $p_m$ are represented by the features that are invariant across the source and target domains.

Formally, the objective function of $D_t$, $D_s$, and $G$ is described by Equation \ref{eq:middle_GAN}. 

\begin{equation} 
\label{eq:middle_GAN}
\begin{aligned}
& \underset{G}{\min} \; \underset{D_s, \: D_t}{\max} \; V(G,D_s, D_t) \\
&= \mathds{E}_{{x}_s \sim p_s({x}_s)} [\log(D({x}_s))] 
+  \mathds{E}_{{z} \sim q({z})} [\log(1-D_s(G(z)))] \\
&+\mathds{E}_{{x}_t \sim p_t({x}_t)} [\log(D({x}_t))]+ \mathds{E}_{{z} \sim q({z})} [\log(1-D_t(G(z)))]
\end{aligned}
\end{equation}

\begin{figure}
\centering
\includegraphics[width=0.5\textwidth]{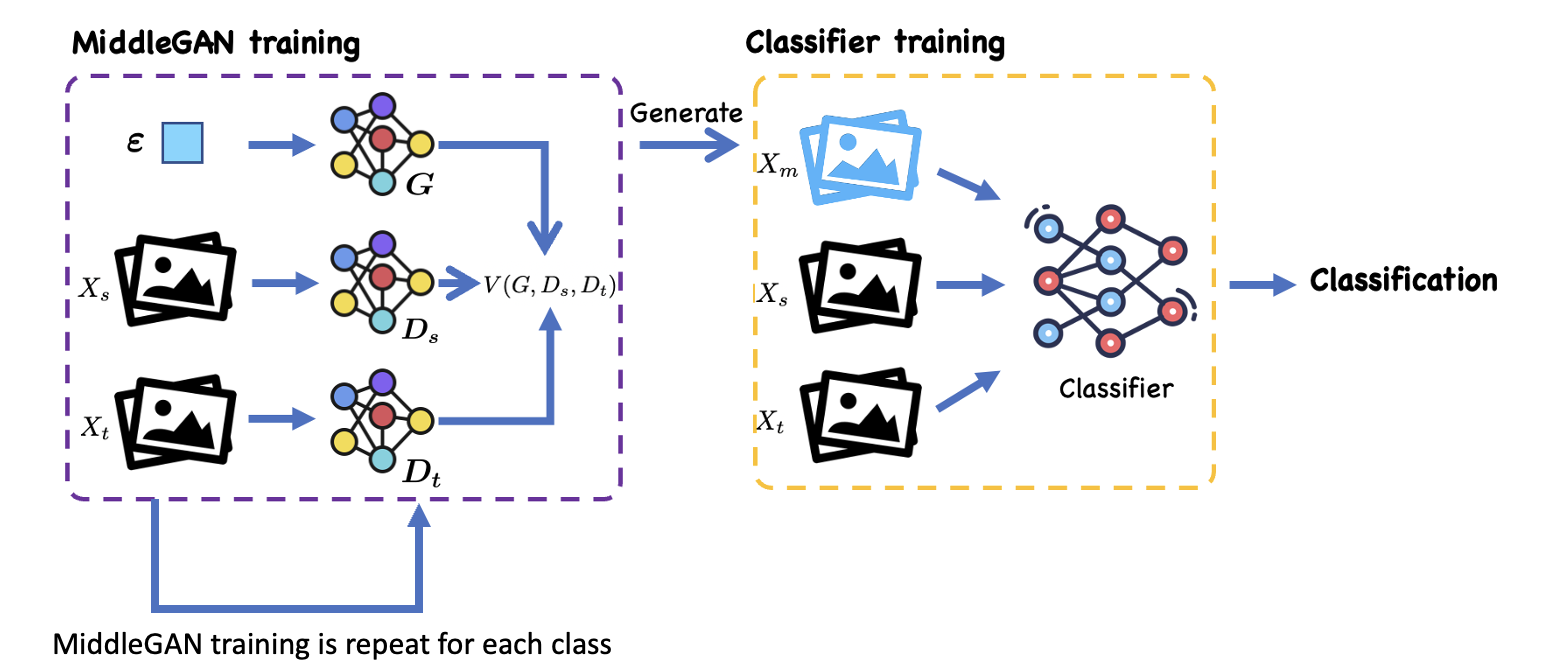}
\caption{In this Figure we describe how to use the MiddleGAN to generate fake, domain agnostic samples and how to use those domain agnostic samples to train the classifier that eventually performs the classification task on the target domain. }
\label{fig:middlegan_flowchart}
\end{figure}

In the previous paragraphs we have described how to generate samples that are similar to both the source samples and the target samples. Recall that we have argued that we will feed these samples during training to the classifier that performs the final classification task on the target domain (in a supervised fashion). In this case, how do we obtain the labels of the generated samples? The labels of the generated data is the same as the labels of the source and target samples that are used to generate them. In other words, only source samples of a particular class and target samples of that particular class get to be used to generate fake samples of this class. We repeat the generation process for all classes in the source and target domains to generate fake samples. 

Note that in the setting of unsupervised domain adaptation, the labels of the target domain are not available. How do we obtain those labels to train the generator and classifier? We propose to use an existing unsupervised domain adaptation algorithm, Fixbi, to obtain psedo-labels for the target domains. Then, we use the psedo-labels of the target domain, together with the labels of the source domain, to train the generator and the classifier.


Figure \ref{fig:middlegan_flowchart} shows the flowchart of how to use the generated domain agnostic samples to train the classifier that eventually performs classification on the target domain. In Figure \ref{fig:middlegan_flowchart}, $\epsilon$ is noise, $X_s$ and $X_t$ are the source samples and target samples respectively, and $X_m$ is the generated samples by $G$. During the training of MiddleGAN, the three neural networks $G$, $D_s$, and $D_t$ engaged in a minimax game. The process of MiddleGAN training (in the purple box) is repeated for all classes in the source domain (and the target domain). During the training of the classifier, both $X_s$, $X_t$, and $X_m$ and their labels are used.

\subsection{Theoretical Results}
We first discuss the two discriminators $D_s$ and $D_t$ given a fixed $G$. We propose Theorem \ref{th:optimal_ds_dt} regarding the optimal values for $D_s$ and $D_t$, represented as $D^{*}_s$ and $D^{*}_t$

\begin{restatable}[]{theorem}{dsdt}
\label{th:optimal_ds_dt}
Given $p_m$, the distribution of samples generated by a fixed generator $G$, the optimal values for the parameters of $D_s$ and $D_t$ are $ D^{*}_s = \frac{p_s}{p_s + p_m} $ and $ D^{*}_t = \frac{p_t}{p_t + p_m} $.
\end{restatable}


\begin{proof}
The value objective $V(G, D_s, D_t)$ can be expanded.

\begin{equation}
\begin{aligned}
V(G, D_s, D_t) & = \int_{x_s} p_s(x_s) \log(D_s(x_s)) dx_s \\
&+ \int_{x_t} p_t(x_t) \log(D_t(x_t)) dx_t \\
&+ \int_z q(z) \log(1-D_s(G(z)))dz \\
&+ \int_z q(z) \log(1-D_t(G(z)))dz \\
&= \int_{x_s} p_s(x_s) \log(D_s(x_s)) dx_s \\
&+ p_m(x_s)\log(1-D_s(x_s))dx_s \\
&+ \int_{x_t} p_t(x_t) \log(D_t(x_t)) \\
&+ p_m(x_t)\log(1-D_s(x_t))dx_t \\
\end{aligned}
\end{equation}

We observe that $p_s$, $p_t$ and $p_m$ belong in $\mathds{R}$. For the source discriminator $D_s$, any pair of $p_s$ and $p_m$ in the form of $p_s \: \log(y) + p_m \: (1- \log(y))$, $p_s \: \log(y) + p_m \: (1- \log(y))$ achieves its maximum value at $\frac{p_s}{p_s + p_m}$ \cite{goodfellow2014generative}. Similarly, for the target discriminator $D_t$, any pair of $p_t$ and $p_m$ in the form of $p_t \: \log(y) + p_m \: (1- \log(y))$, $p_t \: \log(y) + p_m \: (1- \log(y))$ achieves its maximum value at $\frac{p_t}{p_t + p_m}$.
\end{proof}

Now we bring forth Theorem \ref{th:g} which proposes that there exists an optimal solution for the parameters of not only $D_s$ and $D_t$, but also $G$.
\begin{restatable}[]{theorem}{globalminimum}
\label{th:g}
There exists a global minimum for the virtual training criterion $C(G)$ defined as 
\begin{equation}
    C(G) = \underset{D_s, \: D_t}{max} V(G, D_s, D_t).
\end{equation}
In other words, there exists an optimal solution for the parameters of the generator $G$.
\end{restatable}


\begin{proof}
Goodfellow et al. \cite{goodfellow2014generative} have proved that, in the original GAN where there is only one discriminator $D$ and one generator $G$, the virtual training criterion can be written as the following:

\begin{equation}
\centering
\begin{aligned}
C_{original}(G) &= \underset{D}{max} \; V(G, D) \\
&= -log(4) + KL(p\parallel \frac{p+p_m}{2}) + KL(p_m\parallel \frac{p+p_m}{2})
\end{aligned}
\end{equation}
in which $p$ is the distribution of the real samples and $p_m$ is the distribution of generated fake samples, and KL is the Kullback–Leibler divergence. With two discriminators, our virtual training criterion $C(G)$ can be rewritten as:

\begin{equation}
\label{eq:C_G_JSD}
    \begin{aligned}
    C(G) =& -log(4) + KL(p_s\parallel \frac{p_s+p_m}{2}) + KL(p_m\parallel \frac{p_s+p_m}{2}) \\
    & -log(4) + KL(p_t\parallel \frac{p_t+p_m}{2}) + KL(p_m\parallel \frac{p_t+p_m}{2}) \\
    =& -2log(4) + 2JSD(p_s \parallel p_m) + 2JSD(p_t \parallel p_m)
    \end{aligned}
\end{equation}

In Equation \ref{eq:C_G_JSD}, JSD is the  Jensen–Shannon divergence. To find the global minimum, $M(G)$, we want to obtain

\begin{equation}
\label{eq:C_G_JSD_centroid}
    \begin{aligned}
    M(G) &= \underset{p_m}{argmin} -2log(4) + 2JSD(p_s \parallel p_m) + 2JSD(p_t \parallel p_m)
    \end{aligned}
\end{equation}
We observe in Equation \ref{eq:C_G_JSD_centroid} that we are looking for the optimal value  of the JSD centroid defined as $Centroid^{*} = arg \;\underset{Q}{min} \sum_{i=1}^{n} JSD (P_i \parallel Q)$ in which $P_i$ and $Q$ are distributions.
We can see that the generator is essentially looking for the JSD controid of the source domain distribution $p_s$ and the target domain distribution $p_t$. The convexity of the problem has been proved in \citep{nielsen2020generalization}. 
\end{proof}

\subsection{Guaranteed Domain Agnosticism of Generated Samples}
\label{sec:feat_invariance}
Are the samples generated by the MiddleGAN similar to both the source and the target domains? In this section we use two examples to show that fake samples in the distribution $p_m$ are similar to both the samples in the source and target distributions $p_s$ and $p_t$. 

To test if the generated samples are indeed similar to both the source and target samples (domain agnostic), we propose a simple, but effective way to attest it. We treat the original, unaltered MNIST \cite{deng2012mnist} as the source domain. MNIST is a dataset that contains pictures of 10 classes of handwritten digits. For the target domain, we alter the MNIST dataset by rotating each sample 180 degrees. Then, we use the MiddleGAN to generate the intermediate samples. After obtaining the fake samples, we perform the first round of an experiment by training a classifier (Inception v3 with the last layer replaced to have 10 neurons to correspond to ten classes of handwritten digits) on the combination of the training sets of both the source and the target domains as well as the fake samples. We achieve an accuracy of 99.4\% on the source domain's testing set and an accuracy of 99.4\% on the target domain's testing set (Accuracy is calculated in terms of whether the classifier correctly classifies a sample that is of one of the ten classes of handwritten digits). We use a learning rate of 0.0002 and the Adam optimizer and train 5 epochs. Note that the difference between the source and target domains in the first round of the experiments is only caused by the direction of the MNIST samples. To demonstrate if the fake samples are robust to the difference (i.e. domain agnostic), we rotate those fake samples by 180 degree as well. Then, we train a classifier with the same structure using the same hyperparameters including the learning rate, the optimizer, and the training epochs. Then, we train the classifier on the combination of the training sets of both the source and the target domains as well as the \textbf{rotated} fake samples. We have achieved an accuracy of 99.4\% on the source testing set and 99.3\% on the target testing set.

\begin{table}[h]
    \centering
    
    \begin{tabular}{ ccc } 
    \hline
                & Source Acc.     & Target Acc.\\
    \hline

    Upright fake samples     & 99.4\%    & 99.4\%  \\ 
    Rotated fake samples     & 98.9\%    & 99.3\% \\

    \hline
    \\
    \end{tabular}
    \caption{The results from the two rounds of experiments. There is no significant change to the performance measured in accuracy on both the source and target's testing sets.}

    \label{table:fake_robust}
\end{table}

From Table \ref{table:fake_robust} we conclude that there is no significant change to the performance of the classifier, despite that one's training samples contain only upright fake samples and the other's training samples contain only rotated fake samples. \textbf{This indicates that the fake samples are domain agnostic because whether or not we rotate them it makes no difference.}

We have included another example to demonstrate that the fake samples generated by the MiddleGAN are similar to both the source and the target domains. Figure \ref{fig:women_men_middle} contains three subfigures. The first subfigure (from CelebA Dataset \cite{liu2015faceattributes}) contains 64 real samples of cis gender women (the source domain). The second subfigure (from CelebA Dataset \cite{liu2015faceattributes}) contains 64 samples of cis gender men (the target domain). The third subfigure contains 64 samples of fake samples (generated by the MiddleGAN) that visually have both the characteristics of femininity and masculinity. Therefore, it is attested visually that the fake samples generated by MiddleGAN has the similarity of both the source and the target domains. Upon the inspection by two human inspectors, both agree that the generated samples have both feminine and masculine features. In other words, the generated samples are similar to both the source and target domains.

\begin{figure*}
\centering
\includegraphics[width=1\textwidth]{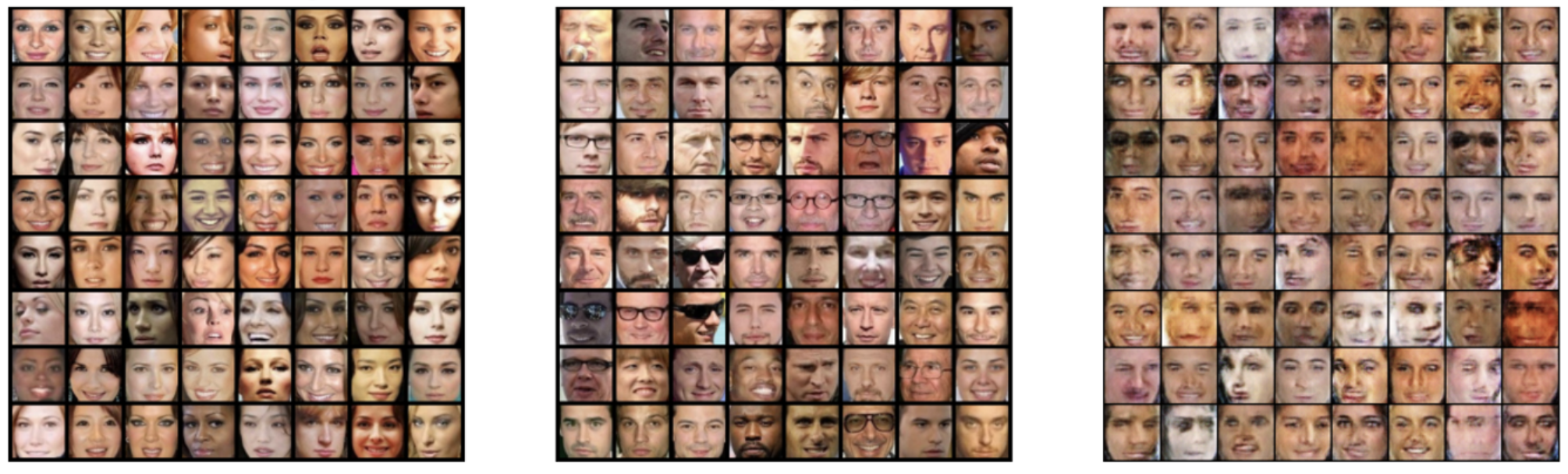}
\caption{The left figure is from the source, real samples of cis gender women. The middle figure is from the target, real samples of cis gender men. The right samples are fake samples generated by MiddleGAN. As we can observe, the fake samples contain both feminine and masculine facial features.}
\label{fig:women_men_middle}
\end{figure*}
\section{Evaluation}

\begin{table*}[h]
    \centering
    \begin{tabular}{ ccccccc|c } 
    
    \hline
    Algorithm   &  A$\rightarrow$W &  A$\rightarrow$D &  D$\rightarrow$W &  D$\rightarrow$A &  W$\rightarrow$A &  W$\rightarrow$D  & Avg    \\
    \hline
    ResNet-50 \cite{he2016deep}    &   68.4\%&            68.9\%&             96.7\%&         62.5\%&         60.7\%&   99.3\%&           76.1\%\\
    DANN \cite{ganin2015unsupervised}         &   82.0\%&            79.7\%&             96.9\%&         68.2\%&         67.4\%&   99.1\%&           82.2\%\\
    MSTN \cite{xie2018learning}        &   91.3\%&            90.4\%&             98.9\%&         72.7\%&         65.6\%&   \textbf{100\%}&   86.5\%\\
    CDAN+E \cite{long2018conditional}      &   94.1\%&            92.9\%&             98.6\%&         71.0\%&         69.3\%&   \textbf{100\%}&   87.7\%\\
    DMRL \cite{wu2020dual}         &   90.8\%&            93.4\%&             99.0\%&         73.0\%&         71.2\%&   \textbf{100\%}&   87.9\%\\
    SymNets \cite{zhang2019domain}      &   90.8\%&            93.9\%&             98.8\%&         74.6\%&         72.5\%&   \textbf{100\%}&   88.4\%\\
    GSDA \cite{hu2020unsupervised}         &   95.7\%&            94.8\%&             99.1\%&         73.5\%&         74.9\%&   \textbf{100\%}&   89.7\%\\
    CAN \cite{kang2019contrastive}         &   94.5\%&            95.0\%&             99.1\%&         78.0\%&         77.0\%&   99.8\%&           90.6\%\\
    SRDC \cite{tang2020unsupervised}        &   95.7\%&            \textbf{95.8\%}&    99.2\%&         76.7\%&         77.1\%&   \textbf{100\%}&   90.8\%\\
    RSDA-MSTN \cite{gu2020spherical}    &   \textbf{96.1\%}&   \textbf{95.8\%}&    99.3\%&         77.4\%&         78.9\%&   \textbf{100\%}&   91.1\%\\
    FixBi \cite{na2021fixbi}       &   \textbf{96.1\%}&   95.0\%&             99.3\%&         78.7\%&         79.4\%&   \textbf{100\%}&   91.4\%\\
    \hline
    MiddleGAN    &   92.4\%&            94.1\%&             \textbf{100\%}&   \textbf{84.9\%}&   \textbf{83.5\%}&   \textbf{100\%}&   \textbf{92.4\%}\\

    \hline
    \\
    \end{tabular}
    \caption{The results on the domain adaptation tasks among the three domains in the dataset Office-31. The metric is accuracy.}

    \label{table:eval_office_31}
\end{table*}

\begin{table*}[h]
    \centering
    \scalebox{0.65}{
    \begin{tabular}{ ccccccccccccc|c } 
    \hline
    Algorithm   & Pr$\rightarrow$Ar & Ar$\rightarrow$Pr & Cl$\rightarrow$Ar & Ar$\rightarrow$Cl & Rw$\rightarrow$Ar & Ar$\rightarrow$Rw & Pr$\rightarrow$Cl & Cl$\rightarrow$Pr& Rw $\rightarrow$ Pr & Pr$\rightarrow$Rw & Rw$\rightarrow$Cl & Cl$\rightarrow$Rw  & Avg    \\
    \hline
    ResNet-50 \cite{he2016deep} & 38.5\% & 50\% & 37.4\% &  34.9\% &  53.9\% &  58\% &  31.2\% &  41.9\% & 59.9\% &  60.4\% &  41.2\% &  46.2\% & 46.1\% \\
    
    DANN \cite{ganin2015unsupervised}        &  41.6\% &  59.3\% &  47.0\% &  45.6\% &  63.2\% &  70.1\% &  43.7\% &  58.5\% &  76.8\% &  68.5\% &  51.8\% &  60.9\% & 57.6\%\\
    CDAN \cite{long2018conditional}       &  55.6\% &  69.3\% &  54.4\% &  49.0\% &  68.4\% &  74.5\% &  48.3\% &  66.0\% &  80.5\% &  75.9\% &  55.4\% &  68.4\% & 63.8\%\\
    MSTN \cite{xie2018learning}       &  61.4\% &  70.3\% &  60.4\% &  49.8\% &  70.9\% &  76.3\% &  48.9\% &  68.5\% &  81.1\% &  75.7\% &  55.0\% &  69.6\% & 65.7\%\\
    SymNets \cite{zhang2019domain}     &  63.6\% &  72.9\% &  64.2\% &  47.7\% &  73.8\% &  78.5\% &  47.6\% &  71.3\% &  82.6\% &  79.4\% &  50.8\% &  74.2\% & 67.2\% \\
    GSDA \cite{hu2020unsupervised}         &  65.0\% &  76.1\% &  65.4\% &  61.3\% &  72.2\% &  79.4\% &  53.2\% &  73.3\% &  83.1\% &  80.0\% &  60.6\% &  74.3\% & 70.3\%\\
    GVB-GD \cite{cui2020gradually}      &  65.2\% &  74.7\% &  64.6\% &  57.0\% &  74.6\% &  79.8\% &  55.1\% &  74.1\% &  84.3\% &  81.0\% &  59.7\% &  74.6\% & 70.4\%\\
    RSDA-MSTN \cite{gu2020spherical}   &  67.9\% &  77.7\% &  66.4\% &  53.2\% &  75.8\% &  \textbf{81.3\%} &  53.0\% &  74.0\% &  85.4\% &  \textbf{82.0\%} &  57.8\% &  76.5\% & 70.9\%\\
    SRDC \cite{tang2020unsupervised}        &  \textbf{68.7\%} &  76.3\% &  \textbf{69.5\%} &  52.3\% &  76.3\% &  81.0\% &  53.8\% &  76.2\% &  85.0\% &  81.7\% &  57.1\% &  78.0\% & 71.3\%\\
    Fixbi \cite{na2021fixbi}      &  65.8\% &  77.3\% &  67.7\% &  58.1\% &  \textbf{76.4\%} &  80.4\% &  57.9\% &  79.5\% &  \textbf{86.7\%} &  81.7\% &  62.9\% &  \textbf{78.1\%} & 72.7\%\\

    \hline
    MiddleGAN   & 65.0\% & \textbf{86.9\%} & 63.3\% & \textbf{78.2\%} & 66.2\% & 76.8\% & \textbf{73.8\%} & \textbf{86.4\%} & 86.1\% & 71.5\% & \textbf{75.2\%} & 73.7\% & \textbf{75.3\%} \\
    \hline
    \\
    \end{tabular}
    }
    \caption{The results on the domain adaptation tasks among the four domains in the dataset Office-Home. The metric is accuracy.}

    \label{table:eval_office_home}
\end{table*}

\begin{table*}[h]
    \centering
    \begin{tabular}{ ccc } 
    \hline
    Algorithm   & SVHN $\rightarrow$ MNIST &  MNIST $\rightarrow$ SVHN     \\
    \hline

    Source Only \cite{french2017self}       &  66.5\%    & 25.4\%  \\ 
    Reverse Grad \cite{bousmalis2017unsupervised}         & 73.9\%     &35.6\% \\
    DCRN \cite{bousmalis2016domain}                 & 81.9\%     & 40.0\% \\
    ADDA \cite{tzeng2017adversarial}                &76.0\%     & - \\
    ATT \cite{ganin2015unsupervised}       & 86.2\%           & 52.8\% \\
    SBADA-GAN \cite{ghifary2016deep}           &76.1\%     & 61.0\% \\
    Mean Teacher \cite{french2017self}        &99.2\%     & 97.0\% \\
    \hline
    MiddleGAN           &\textbf{99.5\%}     & \textbf{99.9\%} \\

    \hline
    \\
    \end{tabular}
    \caption{The results on the domain adaptation task of SVHN $\rightarrow$ MNIST and  MNIST $\rightarrow$ SVHN.}

    \label{table:eval_mnist_svhn}
\end{table*}

\label{sec:evaluation}
In this section, we evaluate the MiddleGAN on the following tasks: CIFAR-10 $\leftrightarrow$ STL-10 (two tasks), MNIST $\leftrightarrow$ USPS (two tasks), MNIST $\leftrightarrow$ SVHH (two tasks), and on two domain adaptation benchmarks Office-31 and Office-Home which contain three domains and four domains, respectively. Therefore, there are 6 domain adaptation tasks derived from Office-31 and 12 domain adaptation tasks derived from Office-Home. On all 24 tasks that we evaluate, MiddleGAN outperforms the state-of-the-art by up to 20.1\% on certain benchmarks.

\subsection{Setups}

\subsubsection{Datasets}
The following datasets are used to evaluate the MiddleGAN.

\textbf{CIFAR-10} \cite{cifar10} contains 10 classes of images that are 32 $\times$ 32 pixels in size. It is a fairly large dataset; each of its classes has 6000 images. Its training set contains 50,000 images and its testing set contains 10,000 images. Accuracy on its testing set is calculated in terms of whether the classifier correctly classifies a sample that is of one of the ten classes of images.

\textbf{STL-10} \cite{coates2011analysis} contains 10 classes of images that are 96 $\times$ 96 pixels in size. It is different from CIFAR-10 by one class. For each of its classes, there are 500 training samples, and 800 testing samples. Accuracy on its testing set is calculated in terms of if the classifier correctly classifies a sample that is one of the ten classes of images. 

\textbf{MNIST} \cite{lecun1998gradient} contains 10 classes of handwritten digits. They are 28 $\times$ 28 pixels in size. There are 60,000 training samples and 10,000 testing samples. Accuracy on its testing set is calculated in terms of if the classifier correctly classifies a sample that is of one of the ten handwritten digits.

\textbf{USPS} \cite{uspsdataset} contains 10 classes of handwritten digits obtained via scanning the envelopes from the USPS. There are 9298 images in total of the 10 classes, and each of them is of size 16 $\times$ 16 pixels. The samples are in grayscale. We have converted it to RGB. Accuracy on its testing set is calculated in terms of if the classifier correctly classifies a sample that is one of the ten handwritten digits.

\textbf{SVHN} \cite{netzer2011reading} stands for Street View House Numbers. It contains 10 classes of digits obtained by street view cameras. There are 600,000 samples of printed images of size 32 $\times$ 32 pixels. Accuracy on its testing set is calculated in terms of whether the classifier correctly classifies a sample that is of one of the ten street view digits.

\textbf{Office-31} \cite{saenko2010adapting} has three domains: Amazon (A), Dslr (D), and Webcam (W). Each domain contains 31 classes of office objects such as  projectors and rulers. In total, there are 4,110 images. Six domain adaptation tasks can be formed from the Office-31 dataset and we evaluate the MiddleGAN against state-of-the-art solutions on all of the domain adaptation tasks. Accuracy on its testing set is calculated in terms of if the classifier correctly classifies a sample that is one of the 31 object classes.

\textbf{Office-Home} \cite{venkateswara2017deep} has four domains: Art (Ar), Clipart (Cl), Real World (Rw), and Product (Pr). Each domain contains 65 classes of objects that can be found in an office or a home, such as flowers and bikes. In total, there are  15,500  images. Twelve domain adaptation tasks can be formed from the Office-Home dataset and we evaluate the MiddleGAN against state-of-the-art solutions on all of the domain adaptation tasks. Accuracy on its testing set is calculated in terms of if the classifier correctly classifies a sample that is one of the 65 object classes.

\subsubsection{Implementation details}
On all 12 domain adaptation tasks, our source discriminator $D_s$, target discriminator $D_t$, and generator $G$ share the same structures. For $D_s$ and $D_t$, we have 5 2D convolutional layers followed by Leaky ReLu layers with a negative slope of 0.2. After each of the 2nd, 3rd, and 4th 2D convolutional layers, a 2D batch norm layer is added. The activation function is Sigmoid. The learning rate of the Adam optimizers for both $D_s$ and $D_t$ are 0.0002. For the generator $G$, its structure contains 5 transposed 2D convolutional layers. After each of the 1st, 2nd, 3rd, and 4th layers, a 2D batch norm layer is added. The activation function is tanh. The learning rate of the Adam optimizer for $G$ is 0.0002. The structures of the discriminators and generator are based on the DCGAN \cite{radford2015unsupervised} (Note that there is only one discriminator in DCGAN and we use the DCGAN's discriminator's structure for both our discriminators). There is no weight decay for any of the three neural nets. Regarding the final classifier trained on the combination of the source training set, the target training set, and the fake images, its architecture is Inception v3 \cite{szegedy2016rethinking, szegedy2017inception}. We train on a NVIDIA A100 GPU. For each domain adaptation task, the number of generated images is empirically determined that result in the final classifier to give the best performance measured in accuracy scores. 

\subsection{CIFAR-10 $\leftrightarrow$ STL-10}
Table \ref{table:eval_cifar_stl} demonstrates the comparison of the MiddleGAN against 5 state-of-the-art baselines in terms of accuracy: VADA \cite{shu2018dirt}, IIMT \cite{yan2020improve}, Enforced Transfer \cite{gao2022enforced}, SE \cite{french2017self} and SEMA \cite{zuo2021margin}. The Source Only algorithm indicates the performance of training a classifier on the source domain and directly applies it to the target domain without mitigating the domain shift. On the task of CIFAR-10 $\rightarrow$ STL-10, it outperforms the second best-performing algorithm, the Enforced Transfer, by 3.4\%; on the task of STL-10 $\rightarrow$ CIFAR-10, it outperforms the second best-performing algorithm, SEMA, by 12.1\%. The superiority of the MiddleGAN suggests that the fake samples are invariant to domain shift.

\begin{table*}[h]
    \centering
    \begin{tabular}{ ccc } 
    \hline
    Algorithm   & CIFAR-10 $\rightarrow$ STL-10     & STL-10 $\rightarrow$ CIFAR-10\\
    \hline

    Source Only \cite{yan2020improve}       & 75.9\%    & 61.8\%  \\ 
    VADA \cite{shu2018dirt}                 & 80.0\%    & 73.5\% \\
    IIMT \cite{yan2020improve}              & 83.1\%    & 81.6\% \\
    Enforced Transfer \cite{gao2022enforced}     & 86.1\%    & - \\
    SE \cite{french2017self}                & 76.3\%    &83.9\% \\
    SEMA \cite{zuo2021margin}               & 78.7\%    & 86.6\% \\
    \hline
    MiddleGAN                             & \textbf{89.5\%} & \textbf{98.7\%}   \\

    \hline
    \\
    \end{tabular}
    \caption{The results on the domain adaptation task of CIFAR-10 $\rightarrow$ STL-10 and  STL-10 $\rightarrow$ CIFAR-10 of 5 state-of-the-art domain adaptation algorithms and the MiddleGAN. On both tasks, we outperform the second best-performing algorithms by a large margin (3.4\% on CIFAR-10 $\rightarrow$ STL-10 and 12.1\% on STL-10 $\rightarrow$ CIFAR-10), which demonstrates the superiority of the MiddleGAN. The metric is accuracy.}

    \label{table:eval_cifar_stl}
\end{table*}

\subsection{Office-31}
In Table \ref{table:eval_office_31}, we compare the MiddleGAN against ResNet-50 and 10 other state-of-the-art domain adaptation algorithms. Again, the metric is accuracy. Out of the six domain adaptation tasks, we achieve the state-of-the-art performance on three of them. However, note that our improvement over the state-of-the-art algorithms is very significant: On the task D $\rightarrow$ A, we improve over the second best-performing algorithm Fixbi by 6.2\%. On tasks that we do not outperform the state-of-the-art, the difference between the MiddleGAN's performance and the state-of-the-art's performance is usually small. For example, on the task of A $\rightarrow$ D, the state-of-the-art performance is 95.8\% and we are only 1.7\% off. Overall, the MiddleGAN achieves state-of-the-art performance on average, outperforming all the other baselines in comparison.

\subsection{Office-Home}
In Table \ref{table:eval_office_home}, we compare the MiddleGAN against Resnet-50 and 9 state-of-the-art domain adaptation algorithms on the Office-Home dataset. Since there are four sub-domains in the Office-Home dataset, there are in total 12 domain adaptation tasks to be done. Out of the 12 domain adaptation algorithms, the MiddleGAN achieves state-of-the-art performance on 5 of them. When the MiddleGAN achieves state-of-the-art performance on a domain adaptation task, it usually outperforms the second best-performing algorithm by a large margin. For example, on the task of Pr $\rightarrow$ Cl, we outperform the second best-performing algorithm Fixbi by 20.1\%. Overall, the MiddleGAN achieves state-of-the-art performance on average, which is an accuracy score of 75.3\%, 2.6\% higher than the second best-performing algorithm, Fixbi.

\subsection{MNIST $\leftrightarrow$ SVHN}
In Table \ref{table:eval_mnist_svhn} we compare the MiddleGAN against six state-of-the-art algorithms and we observe that the MiddleGAN achieves the new state-of-the-art performance on both SVHN $\rightarrow$ MNIST and MNIST $\rightarrow$ SVHN: on the first task, it achieves an accuracy score of 99.5\% and on the second task an accuracy score of 99.9\%. Both scores are nearly 100\%. Compared to the second best-performing algorithm, the Mean Teacher, the MiddleGAN only achieves an improvement of 0.3\% on the first task. This is because there is not enough room for improvement, considering that the Mean Teacher already achieves an accuracy score of 99.2\% on the first task. On the second task, the second best-performing algorithm the Mean Teacher achieves an accuracy score of 97.0\%, and there is more room for improvement since it is not nearly 100\%. As a result, on the second task, we outperform the Mean Teacher by 2.9\%, a more significant improvement compared to our improvement on the first task.



\section{Conclusion}
Our assumption is that if the generated samples are similar to both the source samples and the target samples, a classifier trained on these domain agnostic samples (and the training samples in the source and target domains) will learn to use domain invariant features to do classification (on the target domain). Based on this idea we propose MiddleGAN, a variation of GAN that generate these domain agnostic samples. We have extended the theory of GAN to prove that there exist optimal solutions for the weights of the two discriminators and one generator in MiddleGAN. We have empirically shown that the generated samples are similar to both the source and target domain samples (domain agnostic). We have conducted extensive evaluations using 24 benchmarks; on the 24 benchmarks, we compare MiddleGAN against various state-of-the-art algorithms and outperform the state-of-the-art by up to 20.1\% on certain benchmarks.

\bibliography{aaai22}

\end{document}